\documentclass[runningheads]{llncs}
\usepackage[T1]{fontenc}
\usepackage{amsmath}
\usepackage{amssymb}
\usepackage{float}
\usepackage{graphicx}
\usepackage{import}
\usepackage{listings}
\usepackage{multicol}
\usepackage{multirow}
\usepackage{tikz}
\usetikzlibrary{arrows.meta, positioning}

\newcommand\myent{%
  \mathrel{\ooalign{\hss$\vdash$\hss\cr%
  \kern1ex\raise0.2ex\hbox{\scalebox{1}{$\circ$}}}}}
\begin{document}
\title{A Variable Occurrence-Centric Framework for Inconsistency Handling \\  (Extended Version)\thanks{This paper is an expanded version of an article accepted for publication in AAAI 2025.}}
%
\author{Yakoub Salhi\orcidID{0000-0003-0100-4428}}
\authorrunning{Yakoub Salhi}
\titlerunning{A Variable Occurrence-Centric Framework for Inconsistency Handling}
%
\institute{Univ. Artois, CNRS, UMR 8188, Centre de Recherche en Informatique de Lens (CRIL), F-62300 Lens, France \\
\email{salhi@cril.fr}}
\maketitle              

\begin{abstract}
In this paper, we introduce  a syntactic framework  for analyzing and handling inconsistencies in propositional bases.  
Our approach focuses on examining  the  relationships between variable occurrences within conflicts. 
We propose two dual concepts: Minimal Inconsistency Relation (MIR) and Maximal Consistency Relation (MCR). Each MIR is a minimal equivalence relation on variable occurrences that results in inconsistency, while each MCR is a maximal equivalence relation designed to prevent inconsistency. 
Notably, MIRs capture  conflicts overlooked  by  minimal inconsistent subsets. 
Using MCRs, we develop a series of non-explosive  inference relations. 
The main strategy involves restoring consistency by modifying the propositional base according to each MCR, followed by employing the classical inference relation to derive conclusions.  Additionally, we propose  an unusual semantics that assigns truth values to variable occurrences instead of the variables themselves. The associated   inference relations are established through Boolean interpretations compatible with the occurrence-based models.
\end{abstract}

\section{Introduction}

Logical formalisms commonly used  to represent knowledge and beliefs, particularly classical logic, adhere to the explosion principle. According to this principle,  any conclusion can be derived from a contradiction, which   makes inconsistent  belief bases non informative. 
This issue is critical in real-world applications, where inconsistencies often arise from  various factors including  noisy data, vagueness, context dependency, uncertainty, and information from multiple sources.  This situation shows the need for analytical tools that can identify the causes of conflicts, ensure recovery of consistency, and enable  reasoning  under inconsistency.  Key approaches for addressing this issue include paraconsistency~\cite{Tanaka13,sep-logic-paraconsistent}, argumentation~\cite{BesnardH08},  belief revision~\cite{Gardenfors1992}, and inconsistency measurement~\cite{HunterK10,Thimm:2018}.

When exploring syntactic\footnote{We use the term ``syntactic" to refer to approaches in the literature that focus on the form of belief  bases, including methods based on consistent subsets and  proof theory. In contrast,   other approaches use new semantics to ensure that even inconsistent belief bases  have models
(e.g., see~\cite{Priest91})} approaches to reasoning under inconsistency, two core concepts are frequently used:  minimal inconsistent subsets (MISes) and maximal consistent subsets (MCSes). An example of this is the inference relation proposed by Rescher and Manor, which determines conclusions from formulas entailed by all MCSes~\cite{Rescher1970}. A number of additional MCS-based inference relations have been proposed in the literature~\cite{Brewka89,BenferhatDP97,benferhat:hal-03300219,KoniecznyMV19}. The fundamental principle behind using MCSes is to minimize alterations to the available information.

A significant limitation of using MISes and MCSes, however, lies in  their inability  to account for the specific syntactic representation of each formula: they do not look inside the formulas. 
For instance, a MIS maintains its status when any formula within it is replaced  by an equivalent one.
This drawback limits the ability to identify critical syntactic details, particularly when conflicts stem from factors such as imprecision or encoding errors. 
This issue is notably highlighted in the context of inconsistency measurement where free formulas can influence the inconsistency degree~\cite{Thimm:2018,Besnard14}.
To illustrate, consider the propositional  base $K=\{p\wedge q, \neg p\wedge r, \neg q\vee\neg r\}$ provided in~\cite{Besnard14}.  Here, there is a single MIS $\{p\wedge q, \neg p\wedge r\}$;  neglecting the presence of $q$ and $r$ in these formulas  fails to capture the conflict involving $q$, $r$ and $\neg q\vee\neg r$.

In this article, we introduce new syntactic  concepts for analyzing and handling  propositional inconsistency.
Our approach particularly  focuses on the role of variable occurrences within conflicts. Initially, it distinguishes between occurrences of the same variable, then identifies 
  equivalences between these occurrences that contribute to conflicts. 
This approach allows us to capture the interactions of variable occurrences  in inconsistent propositional  bases and provides a new strategy for restoring consistency.

We first  introduce a concept called Minimal Inconsistency Relation (MIR), which corresponds to  a minimal equivalence relation on variable occurrences that leads to inconsistency.
For instance, the conflict between $q$, $r$ and $\neg q\vee\neg r$  in the previous propositional  base $K$ is an MIR: the equivalence relation pairs the  occurrences of $q$ as well as the occurrences of $r$. 

Beyond the interactions between variable occurrences,  MIRs also highlight additional aspects of the captured conflicts. Specifically, they reveal the variables implicated and the specific formulas they are part of, given that each occurrence appears in  a unique formula.  In particular, when we focus on the formulas associated with MIRs, we capture at least as many conflicts as MISes, and generally more.

We also introduce a dual  concept called Maximal  Consistency  Relation (MCR), which corresponds to  a maximal equivalence relation on variable occurrences constructed  to avoid inconsistency.  We provide  a duality property  between MIRs and MCRs, analogous to the minimal hitting set duality observed between MISes and the complements of MCSes~\cite{Reiter87,BaileyS05}. 

Using MCRs, we introduce  several non-explosive inference  relations. 
Similar to MCS-based approaches, the essential principle behind using MCRs is to maintain the propositional base as close to the original as possible after alterations. 
 However, unlike MCS-based methods, our approach retains every formula and variable present in the original propositional base.
An MCS does not alter any  original formula but excludes some to restore  consistency, whereas an MCR is used to modify  certain formulas without  excluding any  to achieve the same goal.

Additionally, we introduce a new semantics that assigns truth values to  variable occurrences rather than to the variables themselves. This approach mirrors the  technique discussed earlier and can be seen as another strategy for restoring consistency. We specifically demonstrate that each occurrence-based model corresponds uniquely to an MCR. The entailment in our  framework of occurrence-based semantics is realized through Boolean interpretations that are compatible with  these occurrence-based interpretations.


\section{Preliminaries}

The language of classical propositional logic is constructed by inductive definition, starting with a countably infinite set of propositional variables denoted $\textsf{PV}$.
We employ the usual connectives  $\wedge$ and $\neg$ to express logical relationships. The set of well-formed formulas is referred to as $\textsf{PF}$. The  connectives $\vee$, $\rightarrow$ and  $\leftrightarrow$ are defined in terms of $\wedge$ and $\neg$ as usual. 

We use the letters $p$, $q$ and $r$, along with $x$, $y$ and $z$, each possibly modified with subscripts, to denote propositional variables. For formulas, we employ Greek letters such as $\phi$, $\psi$, and $\chi$. When discussing a syntactic entity  $X$ (e.g., a formula or a set of formulas),  the set of variables contained in $X$ is denoted by $\textsf{var}(X)$.

Considering formulas $\phi, \psi_1,\ldots{},\psi_k$, along with   variables $p_1,\ldots{}, p_k$ in $\textsf{var}(\phi)$, we use  
$\phi[p_1/\psi_1,\ldots{}, p_k/\psi_k]$ to denote the result of simultaneously substituting  $p_1,\ldots{},p_k$ with  $\psi_1,\ldots{}, \psi_k$, respectively.

A {\em Boolean interpretation}, or simply interpretation, is a function  $\omega$ that assigns a truth value in $\{0, 1\}$ to each formula  in $\textsf{PF}$ and meets the following conditions:  $\omega(\neg \phi)=1-\omega(\phi)$ and  $\omega(\phi\wedge\psi)= \omega(\phi)\times \omega(\psi)$.

Given an interpretation $\omega$, a variable $p$, and  $v \in \{0,1\}$, 
define $\omega_{|p\mapsto v}$ as the same as $\omega$ except that it assigns $v$ to $p$. For  variables $p_1, p_2, \ldots, p_k$ and  truth values $v_1, v_2, \ldots, v_k$, we use $\omega_{|p_1\mapsto v_1, p_2\mapsto v_2, \ldots, p_k\mapsto v_k}$ as a shorthand notation for $(\cdots((\omega_{|p_1\mapsto v_1})_{|p_2\mapsto v_2})\cdots)_{|p_k\mapsto v_k}$.

An interpretation $\omega$  is said to be a {\em model} of $\phi$, denoted  by $\omega\models\phi$, if and only if $\omega(\phi)=1$. The set of models for  $\phi$ is denoted by $\textsf{mod}(\phi)$.  Conversely,  $\omega$ is a {\em countermodel} of $\phi$, denoted  by $\omega\not\models\phi$, if and only if $\omega(\phi)=0$.

A formula is considered {\em consistent} if it admits at least one model; otherwise, it is referred to as {\em inconsistent}.

A {\em propositional base} (PB) is  a finite subset of  $\textsf{PF}$. 

We say that a PB   $K$ {\em entails} $\phi$, written $K\vdash\phi$, when for any  interpretation  $\omega$, 
if $\omega\models\bigwedge  K$ (the conjunction of all formulas in $K$),  it holds that  $\omega\models \phi$; note that $\omega\models\bigwedge\emptyset$ holds for every interpretation $\omega$.
For cases where  $K$ contains only a single formula $\psi$, we sometimes use $\psi\vdash\phi$.

A \emph{minimal inconsistent subset} (MIS) of a PB  $K$ is a subset $M$ of $K$ where $M$ is inconsistent,  and for every $\phi\in M$, ${M\setminus\{ \phi\}}$ is consistent.
A \emph{maximal consistent subset} (MCS) of $K$  is a subset $M$ of $K$ where $M$ is consistent, and for every $\phi\in K\setminus M$, $M\cup\{ \phi\}$ is inconsistent.


Given a formula $\phi$, we establish the polarity of each subformula occurrence in $\phi$ using the following rules:
\begin{itemize}

\item The polarity of the occurrence $\phi$ is defined as positive.

\item If a subformula occurrence $\psi \wedge \chi$ is positively (resp. negatively) polarized, then both $\psi$ and $\chi$ inherit the same positive (resp. negative) polarity.

\item If a subformula occurrence $\neg \psi$ is positively (resp. negatively) polarized, then $\psi$ is assigned a negative (resp. positive) polarity.
\end{itemize}

A propositional variable that occurs with only one polarity in $\phi$ is called {\em pure} in $\phi$.  


In the proposition below, both the first two properties and the last two properties can be proven simultaneously through mutual induction on the structure of $\phi$.

\begin{proposition}
\label{prop:pv1}
Let $\phi$ be a propositional formula and $p$ a variable in $\textsf{var}(\phi)$. Then,  the following holds: 
\begin{enumerate}
\item If $p$ is a positive pure variable in $\phi$ and $\omega$ is a model of $\phi$, then $\omega_{|p\mapsto 1}$ is a model of $\phi$. 
\item If $p$ is a negative pure variable in $\phi$ and $\omega$ is a counter-model of $\phi$, then $\omega_{|p\mapsto 1}$ is a counter-model of $\phi$. 
\item If $p$ is a negative pure variable in $\phi$ and $\omega$ is a model of $\phi$, then $\omega_{|p\mapsto 0}$ is a model of $\phi$.   
\item If $p$ is a positive pure variable in $\phi$ and $\omega$ is a counter-model of $\phi$, then $\omega_{|p\mapsto 0}$ is a counter-model of $\phi$.   
\end{enumerate}
\end{proposition}
\begin{proof}
We only consider the first two properties, as the proof for the others is similar.
The proof proceeds  by mutual induction on the structure of $\phi$. The base case is straightforward: $\omega_{|p\mapsto 1}(p)= 1$.
The inductive case $\phi=\psi\wedge\chi$ is obtained by applying  the induction hypothesis. For instance, $\omega\models \phi$ implies 
$\omega\models \psi$ and $\omega\models \chi$, leading to $\omega_{|p\mapsto 1}\models \psi$ and $\omega_{|p\mapsto 1}\models \chi$. 
Consider now the inductive case $\phi=\neg\psi$. We prove the first property here; the other follows from a symmetrical argument.
First, note that $p$ is negative in $\psi$ since it is positive in $\phi$.
From  $\omega\models \phi$, we deduce $\omega$ is a counter-model of $\psi$. Thus, applying the induction hypothesis, $\omega_{|p\mapsto 1}$
 also becomes  a counter-model of $\psi$, which results in $\omega_{|p\mapsto 1}\models \phi$. 
\end{proof}

We define $nbOcc(p, \phi)$ as the number of occurrences of variable $p$ in $\phi$. To unambiguously identify and reference each occurrence of  $p$ in $\phi$, we employ integers ranging from $1$ to $nbOcc(p, \phi)$.  Indeed, for a  formula $\phi$ in a  PB, a variable  $p$ occurring in $\phi$, and an index $i$ ranging from $1$ to  $nbOcc(p,\phi)$, $\langle p,\phi, i\rangle$ is used to denote the  $i$th occurrence of $p$ in $\phi$ (ordered from left to right).
Given a PB $K=\{\phi_1,\ldots{}, \phi_n\}$ and a variable $p$ occurring in $K$, we sometimes use $p_i^s$ to denote the $i$th occurrence of $p$ in 
$\phi_1\wedge\cdots{}\wedge\phi_n$, where $s$ is the polarity of this occurrence. 

Additionally,  we use $Occ(K)$  to denote the set of all variable occurrences in $K$,  and $Occ(p,K)$ to denote  the set of occurrences of $p$ in $K$.  We also use $PosOcc(K)$ and $NegOcc(K)$ to represent the positive and negative variable occurrences in $K$, respectively. Similarly, $PosOcc(p, K)$ and $NegOcc(p, K)$ correspond to  the positive and negative occurrences of $p$ in $K$.

For instance, if $K=\{p\wedge q,\neg  p\wedge r, \neg q\vee \neg r\}$, then $Occ(K)= \{p_1^+, p_2^-, q_1^+, q_2^-, r_1^+, r_2^-\}$, 
$PosOcc(K)=\{p_1^+, q_1^+, r_1^+\}$, $NegOcc(K)=\{p_2^-, q_2^-, r_2^-\}$,  and $Occ(p,K)= \{p_1^+, p_2^-\}$.

Let us recall that an equivalence relation  on a set $S$ is a binary relation $\thicksim$ on $S$ that is reflexive, symmetric, and transitive. 
The equivalence class  of $e\in S$ under $\thicksim$ is represented by $[e]$.
The set of all equivalence classes of $S$ by $\thicksim$ is denoted ${S/\thicksim}$.

\section{Variable Occurrence-based Conflicts}

To achieve a fine-grained representation of inconsistency in a PB, we focus on the  occurrences of variables and the conflicts resulting from the relationships between  them.
Our approach becomes particularly valuable when variable occurrences that are supposed to convey identical information end up representing conflicting  data due to various factors such as data errors or interpretive differences.  Consider, for instance,  the case of conflicting medical reports for the same patient but from different laboratories, where each symptom is represented by a propositional variable (each variable occurrence is  linked to a specific laboratory result). 
A symptom like ``high temperature" might be differently reported depending on the laboratory's threshold:  a  temperature of $39^\circ C$ might be classified as high by a laboratory with a threshold of $39^\circ C$, but not by one with a threshold of $40^\circ C$.

A method for resolving all conflicts in a PB involves substituting every occurrence of a variable with a distinct, new variable. Given a PB $K$,  a \emph{C-renaming}  of $K$ is a function $\textsf{R}$ that assigns a distinct, new variable to each occurrence $o$ in $Occ(K)$. 
 We use $\textsf{R}(K)$ to denote the modified version of $K$ achieved by replacing each occurrence $o$ with $\textsf{R}(o)$.  
This concept of C-renaming is applied identically to formulas.

Since the choice of C-renaming does not affect our definitions, we consider  this function to be  fixed and denoted  $\textsf{R}$ for any PB and for any formula.
To enhance readability, we use the letters $p$, $q$, and $r$ to denote the variables occurring in the original PB, and $x$, $y$, and $z$ for those in $\textsf{R}(K)$.

The following proposition  is mainly due to the fact that the new variables associated with the same original variable are equivalent. 
\begin{proposition}
\label{prop:inc1}
A PB $K$ is inconsistent iff the following formula is inconsistent: 
$$\bigwedge \textsf{R}(K)\wedge  \bigwedge_{p\in \textsf{\scriptsize var}(K)}\bigwedge_{o, o'\in Occ(p,K)} (\textsf{R}(o)\leftrightarrow \textsf{R}(o')).$$ 
\end{proposition}

Considering a PB $K$, we define an equivalence relation $\thicksim_c^K$ on $Occ(K)$ as follows: 
$o\thicksim_c^K o'$ iff $\textsf{var}(o)=\textsf{var}(o')$.

Let us now introduce the main concept we use to represent conflicts.
\begin{definition}[Minimal Inconsistency  Relation]
\label{def:ik}
A {\em Minimal Inconsistency  Relation} (MIR) of a PB $K$ is an equivalence relation $\thicksim$ on $Occ(K)$ satisfying the following conditions:
\begin{enumerate} 
\item (Compliance) for all occurrences $o,o'\in Occ(K)$, if ${o\thicksim o'}$, then $\textsf{var}(o)=\textsf{var}(o')$; \label{prop:d11}
\item (Inconsistency) $\textsf{R}(K)\wedge \bigwedge_{(o,o')\in \thicksim} (\textsf{R}(o)\leftrightarrow \textsf{R}(o'))$ is inconsistent; and \label{prop:d12}
\item (Minimality) there exists no equivalence relation $\thicksim'$ on $Occ(K)$ that satisfies Properties (\ref{prop:d11}) and (\ref{prop:d12}), 
and $\thicksim'$ is a proper subset of $\thicksim$ (i.e., $\thicksim'\subsetneq \thicksim$).
\end{enumerate}
\end{definition}

 The set of all MIRs of $K$ is represented by  $\textsf{MIRs}(K)$.

The Compliance condition states that equivalence can only occur between occurrences of the same variable. The Inconsistency condition says that every MIR must result in inconsistency. The Minimality condition guarantees that every MIR is minimal with respect to set inclusion.

Using the Compliance condition, it is evident that  $\thicksim\subseteq \thicksim_c^K$ holds  for any PB $K$ and any MIR $\thicksim$ of $K$.

\begin{example}
\label{example1}
Consider the inconsistent PB $K_1=\{p\wedge q,\neg  p\wedge r, \neg q\vee \neg r\}$ and $\textsf{R}(K_1)=\{x_1 \wedge y_1, \neg x_2\wedge z_1,  \neg y_2\vee \neg  z_2\}$.
The PB $K_1$ admits two  MIRs  $\thicksim_1^i$ and $\thicksim_2^i$: ${Occ(K_1)/\thicksim_1^i}= \{\{p_1^+,p_2^-\} ,\{q_1^+\}, \{q_2^-\},\{r_1^+\},\{r_2^-\}\}$ and 
${Occ(K_1)/\thicksim_2^i}= \{\{p_1^+\},\{p_2^-\} \{q_1^+, q_2^-\},\{r_1^+,r_2^-\}\}$.
\end{example}

The next  theorem is mainly derived  from Proposition~\ref{prop:inc1}.

\begin{theorem}
\label{th:alo1}
A PB is inconsistent iff it admits at least one MIR.
\end{theorem}
\begin{proof}
This result mainly follows from Proposition~\ref{prop:inc1}: a PB $K$ is inconsistent if and only if $\thicksim_c^K$ satisfies Property~2 in the definition of MIR. 

\emph{If Part}. Assume that $K$ admits a MIR $\sim$.
By definition,   $\textsf{R}(K)\wedge \bigwedge_{(o,o')\in \thicksim} (\textsf{R}(o)\leftrightarrow \textsf{R}(o'))$  is inconsistent. This leads to the inconsistency of 
 $\bigwedge \textsf{R}(K)\wedge$ $ \bigwedge_{p\in Var(K)} $ $\bigwedge_{o, o'\in Occ(p,K)} (\textsf{R}(o)\leftrightarrow \textsf{R}(o'))$. Applying Proposition~\ref{prop:inc1}, 
it follows  that $K$ is inconsistent. 
 
\emph{Only If Part}. Suppose that $K$ is inconsistent. 
By Proposition~\ref{prop:inc1}, 
  $\bigwedge \textsf{R}(K)\wedge  \bigwedge_{p\in Var(K)} \bigwedge_{o, o'\in Occ(p,K)} (\textsf{R}(o)\leftrightarrow \textsf{R}(o'))$ is also inconsistent. 
This implies that $\thicksim_c^K$ satisfies the Inconsistency property in the definition of MIR.  Thus, there exists $\thicksim\subseteq \thicksim_c^K$ that satisfies the 
  Inconsistency and Minimality properties. Compliance is also clearly satisfied by $\thicksim$. 
\end{proof}

The proposition below demonstrates that each equivalence class with more than one element includes both positive and negative occurrences.

\begin{proposition}
\label{prop:bpol}
Let $K$ be a PB. 
For every MIR $\thicksim$ of  $K$ and every $C\in {Occ(K)/\thicksim}$ with $|C|\geq 2$,  it holds that
$C\cap PosOcc(K)\neq \emptyset$ and $C\cap NegOcc(K)\neq \emptyset$. 
\end{proposition}
\begin{proof}
Assume for contradiction that there exists an equivalence class
 $C=\{o_1,\ldots{}, o_k\}$ s. t. $k\geq 2$ and $C\cap NegOcc(K)= \emptyset$. Consider refining the equivalence relation to 
 $\thicksim'$, where ${Occ(K)/\thicksim'} =(({Occ(K)/\thicksim})\setminus C)\cup\bigcup_{o\in C}\{\{o\}\}$. Using Minimality,  $\thicksim'$ does not satisfy Inconsistency. Thus 
$\textsf{R}(K)\wedge \bigwedge_{(o,o')\in \thicksim'} (\textsf{R}(o)\leftrightarrow \textsf{R}(o'))$ admits a model $\omega$. Since for every  $o$
in $C$, $R(o)$ is positive in $\textsf{R}(K)$, and applying Proposition~\ref{prop:pv1},  $\omega_{|R(o_1)\mapsto 1,\ldots{}, R(o_k)\mapsto 1}$ is also a model 
of $\textsf{R}(K)$. Consequently,  $\textsf{R}(K)\wedge \bigwedge_{(o,o')\in \thicksim} (\textsf{R}(o)\leftrightarrow \textsf{R}(o'))$ is consistent, leading to a contradiction.  

A similar contradiction arises if we assume $C\cap PosOcc(K)= \emptyset$.  This is obtained by using the truth value $0$ instead of $1$.
\end{proof}


Given a PB $K$ and an equivalence relation $\thicksim$ on $Occ(K)$, define  $PN(\thicksim)$ as  the set 
$\{(o,o')\in PosOcc(K)\times NegOcc(K) : o\thicksim o'\}$.

The following theorem shows  that the core of conflicts fundamentally stems from the interactions between positive and negative occurrences.
It is primarily  a consequence of  Proposition~\ref{prop:bpol}.
\begin{theorem}
Let  $K$ be a PB and  $\thicksim$ an equivalence relation on $Occ(K)$. The relation $\thicksim$ is an MIR iff it satisfies the properties of  Compliance, Inconsistency, and the following property: 
(Minimality-2) for every equivalence relation $\thicksim'$ on $Occ(K)$ that satisfies Compliance and  Inconsistency, $PN(\thicksim')$ is not a proper subset of $PN(\thicksim)$.
\end{theorem}
\begin{proof}
For the \emph{if part}, assume for contradiction that $\thicksim$ is not a MIR, despite satisfying the given conditions. This implies there exists an equivalence relation
$\thicksim'\subsetneq \thicksim$ that satisfies Compliance and Inconsistency. Let  $(o,o')$  be a pair in $\thicksim\setminus\thicksim'$, where clearly $o\neq o'$. By Minimality-2,   
$o$ and $o'$ must have the same polarity.  Consider if both $o$ and $o'$ are positive (the argument is symmetrical for both being negative). According to Proposition~\ref{prop:bpol},
there must be a negative occurrence $o''$ such that $o\thicksim o''$ and $o'\thicksim o''$. By the definition of Minimality-2,  it must hold that  $o\thicksim' o''$ and $o''\thicksim' o'$.
Thus,   by the transitivity of $\thicksim'$, we obtain $(o,o')\in\thicksim'$, which contradicts $(o,o')\in\thicksim\setminus\thicksim'$.

Consider the \emph{only if} part. Assume that $\thicksim$  is a MIR. Let    $\thicksim'$ is another equivalence relation 
that maintains Compliance and Inconsistency. By Minimality,  
$\thicksim'\not\subsetneq \thicksim$ holds. Then, there exists $(o,o')\in {\thicksim'\setminus \thicksim}$. 
If $o$ and $o'$ do not have the same polarity, we obtain $PN(\thicksim')\not\subsetneq PN(\thicksim)$.
Consider the case where $o$ and $o'$ are both positive. Hence, 
by Proposition~\ref{prop:bpol}, there exists a negative occurrence $o''$ s.t.  $(o,o'')$ and $(o',o'')$ belong to $PN(\thicksim')$, leading to $PN(\thicksim')\not\subsetneq PN(\thicksim)$. The case where they are both positive is similar. 
\end{proof}

MIRs can provide a more detailed  representation than MISes even in scenarios where the objective is to identify conflicts within subsets of formulas.

Given a PB $K$ and  an MIR  $\thicksim$ on $Occ(K)$, ${\textsf{form}_{mir}(\thicksim)}$ represents  the set  
$\{{\phi\in K} : {\exists C\in {Occ(K)/\thicksim}}, \exists p\in \textsf{var}(\phi), {\exists i\in \mathbb{N}}\mbox{ s. t. } {|C|\geq 2} \mbox{ and } {\langle p,\phi, i\rangle\in C} \}$.
Informally, $\textsf{form}_{mir}(\thicksim)$ corresponds to  the formulas that contain the occurrences that are paired with $\thicksim$ (excluding   reflexive links).

\begin{definition}[O-MIS]
An O-MIS of a PB $K$ is a subset $M$ of $K$ for which there exists an MIR $\thicksim$ of $K$ such that  
$M=\textsf{form}_{mir}(\thicksim)$.
\end{definition}

Given that $\textsf{form}_{mir}(\thicksim)$ is inconsistent for every MIR $\thicksim$, it follows that every O-MIS includes a MIS. 

\begin{proposition}
Given a PB  $K$, if $M$ is a MIS of $K$, then $M$ is also an O-MIS of $K$.
\end{proposition}
\begin{proof}
According to Theorem~\ref{th:alo1}, $M$ is associated with at least one MIR $\thicksim$. 
Let $\thicksim'$ be the equivalence relation on $Occ(K)$ obtained by extending  $\thicksim$ to pair each occurrence in $K\setminus M$ with only itself. 
Since $\thicksim$ is an MIR of $M$, it follows that $\thicksim'$ is an MIR of $K$. 
We know that $\textsf{form}_{mir}(\thicksim')$ is inconsistent. If $\textsf{form}_{mir}(\thicksim') \subsetneq M$, it would contradict the minimality of $M$. Therefore, we must have
 ${\textsf{form}_{mir}(\thicksim')} = M$, implying that $M$ is an O-MIS.
\end{proof}

\begin{example}
Consider again the PB  $K_1$ from Example~\ref{example1}. With respect to the two possible MIRs, $\thicksim_1^i$ and $\thicksim_2^i$,  $K_1$ admits two O-MISes: $M_1 = \{p \wedge q, \neg p \wedge r\}$ and $M_2 = \{p \wedge q, \neg p \wedge r, \neg q \vee \neg r\}$. However, $K_1$ has only one MIS, namely  $M_1$. 
\end{example}

An interesting application of the MIR and O-MIS concepts is their  use in quantifying inconsistency, similar to how MISes have been employed in this context. Several measures can be defined, for instance,  by focusing on the number of MIRs, their sizes, or the number of equivalence classes.
\section{A Dual Notion: Maximal Consistency  Relation}

This section focuses on a dual notion to the MIR. This notion  is  defined as a maximal equivalence relation on variable occurrences constructed to avoid inconsistency.

\begin{definition}[Maximal Consistency  Relation]
\label{def:ik}
A {\em Maximal Consistency Relation} (MCR) of a PB $K$ is an equivalence relation $\thicksim$ on $Occ(K)$ satisfying the following conditions:
\begin{enumerate} 
\item (Compliance)   for all  occurrences $o,o'\in Occ(K)$, if ${o\thicksim o'}$, then $\textsf{var}(o)=\textsf{var}(o')$;\label{prop:d21}
\item (Consistency) $\bigwedge\textsf{R}(K)\wedge \bigwedge_{(o,o')\in \thicksim} (\textsf{R}(o)\leftrightarrow \textsf{R}(o'))$ is consistent; and \label{prop:d22}
\item (Maximality) there exists no equivalence relation $\thicksim'$ on $Occ(K)$ that satisfies Properties (\ref{prop:d21}) and (\ref{prop:d22}), 
and $\thicksim$ is a proper subset of $\thicksim'$ (i.e., $\thicksim\subsetneq \thicksim'$).
\end{enumerate}
\end{definition}

Let $\textsf{MCRs}(K)$ represent  the set of all MCRs of $K$.

Like MCSes, MCRs are uniquely determined in consistent PBs. 

\begin{proposition}
For every consistent PB $K$,  the relation $\thicksim_c^K$ is the unique MCR of $K$.
\end{proposition}

\begin{example}
Revisiting the inconsistent PB $K_1$ from Example~\ref{example1}, we find that it  admits two MCRs $\thicksim_1^c$ and $\thicksim_2^c$, where 
${Occ(K_1)/\thicksim_1^c}=\{\{p_1^+\},\{p_2^-\}, \{q_1^+, q_2^-\}, \{r_1^+\},$ $\{r_2^-\}\}$ and ${Occ(K_1)/\thicksim_2^c}=\{\{p_1^+\},\{p_2^-\}, \{q_1^+\},\{q_2^-\}, \{r_1^+,r_2^-\}\}$. 
\end{example}

Forgetting is a well-established method for restoring consistency~\cite{Lin1994ForgetI,LangM02,Besnard16}. In particular, Besnard~\cite{Besnard16}, in the context of defining an inconsistency measure, proposes substituting variable occurrences with constants to achieve consistency. The MCR concept can further enhance this perspective by providing a more holistic view.

For example, consider the PB $K = \{p_1 \wedge \cdots \wedge p_n, \neg p_1 \wedge \cdots \wedge \neg p_n\}$.
A forgetting-based approach would produce $2^n$ minimal repairs, since each variable 
$p_i$ could be individually forgotten to resolve the contradiction between $p_i$ 
and $\neg p_i$ (this involves replacing either  $p_i$  or $\neg p_i$ with $\top$). In contrast, our approach yields a single MCR $\thicksim$ with
$Occ(K)/\sim = \{\{p_i\}, \{\neg p_i\} : i = 1,\ldots,n\}$.

Given a PB $K$, an equivalence relation $\thicksim$ on $Occ(K)$, and  $p$ in $\textsf{var}(K)$, 
$\textsf{EqC}(p,\thicksim)$ denotes the set of equivalence classes  in ${Occ(K)/\thicksim}$ that contain occurrences of $p$.

Note that a variable can be associated with at most two equivalence classes in an MCR.
This mainly arises from the observation that in any model $\omega$ of  $\textsf{R}(K)$,   each occurrence is associated with one of two possible truth values, $0$ and $1$.

\begin{proposition}
Let $K$ be a PB. For every MCR $\thicksim$ of $K$ and  every  variable $p\in\textsf{var}(K)$,  the cardinality of $\textsf{EqC}(p,\thicksim)$ is at most two, i.e., 
$|\textsf{EqC}(p,\thicksim)|\leq 2$.
\end{proposition}
\begin{proof}
Assume, for the sake of contradiction, that there exists an MCR $\thicksim$ of $K$ such that $|\textsf{EqC}(p, \thicksim)| \geq 3$. Let $C_1$, $C_2$, and $C_3$ be three distinct equivalence classes in $\textsf{EqC}(p, \thicksim)$.
Given the definition of MCR, the formula $\Psi = \textsf{R}(K) \land \bigwedge_{(o, o') \in \thicksim} (\textsf{R}(o) \leftrightarrow \textsf{R}(o'))$ must be consistent. Therefore, a model $\omega$ of $\Psi$ exists. In this model, for any $i \in \{1, 2, 3\}$ and any $o, o' \in C_i$, it holds that $\omega(\textsf{R}(o)) = \omega(\textsf{R}(o'))$.
Given that there are only two possible truth values, the truth value associated with $C_3$ must match that of either $C_1$ or $C_2$. Assume the truth values for $C_1$ and $C_2$ are identical; the argument for other case is analogous.
Consequently, $\omega$ is also a model of  $\Psi \land \bigwedge_{o \in C_1, o' \in C_3} (\textsf{R}(o) \leftrightarrow \textsf{R}(o'))$, indicating that $\thicksim$ fails to maintain the Maximality property. This leads to a contradiction.
\end{proof}

Consider  the following  complementary  definition of Minimality-2: 
\begin{itemize}
\item (Maximality-2) for every equivalence relation $\thicksim'$ on $Occ(K)$ that satisfies Compliance and  Consistency, $PN(\thicksim)$ is not a proper subset of $PN(\thicksim')$.
\end{itemize}

The goal in the foregoing condition is to maximize the equivalence between positive and negative occurrences. This  is grounded in the understanding that the essence of conflicts mainly  arises from the interactions between positive and negative occurrences.

\begin{definition}[BMCR]
A BMCR of a PB $K$ is an MCR $\thicksim$ of $K$ that satisfies Maximality-2.
\end{definition}

In the definition of MIR, we demonstrated that incorporating Minimality-2 does not affect the concept, as replacing Minimality with Minimality-2 yields the same notion. However, the following example illustrates that MCRs are not always  BMCRs.

\begin{example}
\label{ex:MCR}
Consider the PB $K_2=\{p,\neg p, \neg p\vee q\}$. 
It admits two MCRs $\thicksim_1^c$ and $\thicksim_2^c$, where ${Occ(K_2)/\thicksim_1^c}=\{\{p_1^+, p_3^-\}, \{p_2^-\},\{q_1^+\}\}$ and 
${Occ(K_2)/\thicksim_2^c}=\{\{p_1^+\}, $ $\{p_2^-, p_3^-\},\{q_1^+\}\}$. The unique BMCR is $\thicksim_1^c$ because ${PN(\thicksim_1^c)} =\{(p_1^+, p_3^-)\}$ and
$PN(\thicksim_2^c)=\emptyset$.
\end{example}

The minimal hitting set duality between MISes and the  complements of MCSes asserts that every MIS is a minimal hitting set of the set of all complements  of MCSes,  and vice versa (e.g.,  see~\cite{Reiter87,BaileyS05,LiffitonMPS05}). This property is particularly useful for computing all MISes and MCSes. Here, we show that a similar duality property exists between MIRs and MCRs.

Let $U$ be a set of elements and $S=\{S_1,\ldots{}, S_k\}$ a collection of subsets of $U$. A   {\em hitting set} of $S$ is a set $H\subseteq U$ that intersects with every element
of $S$, i.e., for every $S_i\in S$, $S_i\cap H\neq\emptyset$. A hitting set is said to be {\em minimal} if no proper subset of it can also serve as a hitting set.

To establish  the duality properties, we need some preliminary notions.
\begin{definition}[$H$-Maximality]
\label{maxh}
Let $K$ be a PB and $H \subseteq \thicksim_c^K$. An equivalence relation $\thicksim$ on $Occ(K)$, where $\thicksim \subseteq \thicksim_c^K$, is said to be  $H$-maximal if it meets the following conditions:
(i)~$\thicksim \cap H = \emptyset$, and
(ii)~for any equivalence relation $\thicksim'\subseteq \thicksim_c^K $ with  $\thicksim\subsetneq \thicksim'$,  $\thicksim' \cap H \neq \emptyset$. 
\end{definition}

\begin{definition}[$H$-Minimality]
\label{minh}
Let $K$ be a  PB and $H \subseteq \thicksim_c^K$. An equivalence relation $\thicksim$ on $Occ(K)$, where $\thicksim \subseteq \thicksim_c^K$, is said to be  $H$-minimal if it meets the following conditions:
(i)~$H\subseteq \thicksim$, and
(ii) for any equivalence relation  $\thicksim'\subsetneq \thicksim$, $H\not\subseteq \thicksim'$.
\end{definition}

Alternatively stated,  an equivalence relation $\thicksim\subseteq \thicksim_c^K$ is $H$-maximal if and only if  it  excludes all elements of $H$ and is maximal with respect to set inclusion. 
The  relation $\thicksim$ is $H$-minimal    if and only if  it includes all elements of $H$ and is minimal with respect to set inclusion. 

\begin{definition}[C-MCR]
A C-MCR of a PB $K$ is a relation $\theta$ on $Occ(K)$ such that $\thicksim_c^K\setminus \theta$ is an MCR.
\end{definition}

Let $\textsf{CMCRs}(K)$ represent  the set of all C-MCRs of $K$.

For simplicity, the duality theorem references MCRs in one property and C-MCRs in the other.
\begin{theorem}
Let $K$ be a PB and $\thicksim$ an equivalence relation on $Occ(K)$ s.t. $\thicksim \subseteq \thicksim_c^K$. Then, the following properties hold:
\begin{enumerate}
\item $\thicksim$ is an MCR of $K$ iff there exists a minimal hitting set $H$ of $\textsf{MIRs}(K)$ such that $\thicksim$ is $H$-maximal. \label{prop1}
\item  $\thicksim$ is an MIR if and only if there exists a minimal hitting set $H$ of $\textsf{CMCRs}(K)$ such that $\thicksim$ is  $H$-minimal.\label{prop1}
\end{enumerate}
\end{theorem}

\begin{proof}
We provide a proof for Property 1 only, as the other is supported by a symmetrical proof.
In this proof, $(R)^*$ represents  the transitive and symmetric closure of $R$.

We begin by proving the \emph{if} part.
Let $H$ be a minimal hitting set of $\textsf{MIRs}(K)$ s. t. $\thicksim$ is an $H$-maximal equivalence relation. Note that $(o,o) \notin H$ for any occurrence $o$; otherwise, Property (i) from Definition~\ref{maxh} would be violated, as $\thicksim \cap H \neq \emptyset$.
Assume, for contradiction, that $\thicksim$ does not satisfy Consistency. This would imply that $\textsf{R}(K) \wedge \bigwedge_{(o,o')\in \thicksim} (\textsf{R}(o) \leftrightarrow \textsf{R}(o'))$ is inconsistent, meaning there exists an MIR $\thicksim'$ of $K$ such that $\thicksim' \subseteq \thicksim$. This leads to a contradiction since $\thicksim \cap H = \emptyset$ and there exists $(o,o') \in H \cap \thicksim'$. Thus, $\thicksim$ satisfies Consistency.
Now, suppose $\thicksim$ does not satisfy Maximality. This implies that there exists $(o,o') \in \thicksim_c^K \setminus \thicksim$ s. t. $(\thicksim \cup {(o,o')})^* \cap H=\emptyset$, resulting in a contradiction since $\thicksim$ is $H$-maximal. 

Next, we prove the \emph{only if} part.
Assume $\thicksim$ is an MCR of $K$. Let $H = \thicksim_c^K \setminus \thicksim$ (the C-MCR associated with $\thicksim$). Suppose $H$ is not a hitting set of $\textsf{MIRs}(K)$. Then there exists an MIR $\thicksim'$ of $K$ such that $\thicksim' \subseteq \thicksim$, which leads to a contradiction since $\thicksim$ does not include any MIR.
Let $H'$ be an arbitrary minimal hitting set of $\textsf{MIRs}(K)$ s. t.  $H' \subseteq H$. Suppose $\thicksim$ is not $H'$-maximal. Then there exists $(o,o') \in H$ s.t. $(\thicksim \cup \{(o,o')\})^*$ does not intersect with $H'$. Therefore, $(\thicksim \cup {(o,o')})^*$ satisfies Consistency since it does not include any MIR, leading to a contradiction with  Maximality. 
\end{proof}

\section{MCR-based Inference Relations}

Similar to how MCSes are used to define  inference  relations,  new non-explosive inference relations can be established using  (B)MCRs. 
The key  idea is to restore consistency by modifying  the propositional  base according to  each MCR. This involves assigning a distinct  variable to each equivalence class. Following this assignment, we apply the classical inference relation to derive conclusions.

The main advantage of our approach compared to those based on MCSes is that it retains every formula, including inconsistent ones, and every variable when employing the classical inference relation. In particular, our inference relations, unlike those based on MCSes, do not distinguish between a base and its corresponding formula.

We define an \emph{inference relation} (IR) as a binary relation between PBs and formulas.
  An IR  is considered \emph{more cautious} than another if it is not identical to the other, and  every conclusion entailed by  the first relation can also be entailed by the second~\cite{PinkasL92}.

Considering an MCR  $\thicksim$ of  $K$, we use the following preliminary notions and notations: 
\begin{itemize}
\item  A $\thicksim$-renaming function  is a function $\rho_\thicksim$  that assigns a  distinct propositional variable to each equivalence class in ${Occ(K)\slash\thicksim}$, where for every $p\in\textsf{var}(K)$,   if $C$ is an equivalence  class of $\thicksim$  containing all occurrences of  $p$, then $\rho_\thicksim(C) = p$.  For an  occurrence $o$, we often  write $\rho_\thicksim(o)$ to denote  $\rho_\thicksim([o])$. 
\item  $\rho_\thicksim(K)$ represents  a PB  constructed by replacing each variable occurrence  $o$ with $\rho_\thicksim(o)$.
\item $\lceil\rho_\thicksim\rceil$ is a function that maps each variable $p$ to the set $\{\rho_\thicksim(o) : o \in Occ(p,K)\}$.
\item For a tuple of distinct variables  $S=(p_1,\ldots{}, p_m)$, we define  $\textsf{P}(\rho_\thicksim, S)$ as the set  of tuples $\lceil\rho_\thicksim\rceil (p_1)\times \cdots{}\times \lceil\rho_\thicksim\rceil (p_m)$.
\end{itemize}

 The choice of  the $\thicksim$-renaming function  does not impact any of our IRs. Thus, for every MCR $\thicksim$,  we assume  that this function is fixed and denoted $\rho_\thicksim$.

\begin{example}
\label{ex:MCR2}
Consider again the PB $K_2=\{p,\neg p, \neg p\vee q\}$ from Example~\ref{ex:MCR}. 
Renaming functions associated with ${\thicksim_1^c}$ and  ${\thicksim_2^c}$ can be defined as follows: 
\begin{itemize}
\item  $\rho_{\thicksim_1^c}=\{\{p_1^+, p_3^-\}\mapsto x_1, \{p_2^-\}\mapsto x_2,\{q_1^+\}\mapsto q\}$
\item $\rho_{\thicksim_2^c}=\{\{p_1^+\}\mapsto x_1, \{p_2^-, p_3^-\}\mapsto x_2,\{q_1^+\}\mapsto q\}$. 
\end{itemize}
We obtain $\rho_{\thicksim_1^c}(K_2)=\{x_1, \neg x_2, \neg x_1\vee q\}$ and  $\rho_{\thicksim_2^c}(K)=\{x_1, \neg x_2, \neg x_2\vee q\}$. 
Moreover, we have $\lceil\rho_{\thicksim_1^c}\rceil=\lceil\rho_{\thicksim_2^c}\rceil=\{p\mapsto \{x_1,x_2\}, q\mapsto\{q\}\}$.
Finally, $\textsf{P}(\rho_{\thicksim_1^c}, (p,q))=\textsf{P}(\rho_{\thicksim_2^c}, (p,q))=\{(x_1,q), (x_2,q)\}$.
\end{example}

We now  introduce four MCR-based IRs: 
\begin{itemize}
\item ${K\myent_1 \phi}$ iff for every MCR $\thicksim$ of $K$, there exists a tuple ${(q_1,\ldots{}, q_m)\in \textsf{P}(\rho_\thicksim, S)}$ such that
 $\rho_\thicksim(K)\vdash \phi[p_1/q_1,\ldots{}, p_m/q_m]$.
\item ${K\myent_2 \phi}$ iff for every MCR $\thicksim$ of $K$ and  for every  tuple ${(q_1,\ldots{}, q_m)\in \textsf{P}(\rho_\thicksim, S})$, 
 $\rho_\thicksim(K)\vdash \phi[p_1/q_1,\ldots{}, p_m/q_m]$.
\item ${K\myent_1^B \phi}$ iff for every BMCR $\thicksim$ of $K$, there exists a tuple  ${(q_1,\ldots{}, q_m)\in \textsf{P}(\rho_\thicksim, S)}$ such that
 $\rho_\thicksim(K)\vdash \phi[p_1/q_1,\ldots{}, p_m/q_m]$.
\item ${K\myent_2^B \phi}$ iff for every BMCR $\thicksim$ of $K$ and  for every tuple ${(q_1,\ldots{}, q_m)\in \textsf{P}(\rho_\thicksim, S)}$, 
 $\rho_\thicksim(K)\vdash \phi[p_1/q_1,\ldots{}, p_m/q_m]$.
\end{itemize}
Here,  $S = (p_1, \ldots, p_m)$ represents  a tuple of distinct variables   such that  $\textsf{var}(K) \cap \textsf{var}(\phi) = \{p_1, \ldots, p_m\}$.

The relation $\myent_1$ asserts that a formula $\phi$ is a consequence of  $K$ if, for every MCR $\thicksim$, the PB derived from $K$ by applying a renaming (which assigns a distinct variable to each equivalence class of $\thicksim$) classically entails  at least one version of $\phi$ that is  renamed in a similar way.
In the case of $\myent_2$, we require that all versions of $\phi$ be entailed by the PB after renaming.  The relations $\myent_1^B$ and $\myent_2^B$ are analogous to  $\myent_1$ and $\myent_2$, respectively,  but  use BMCRs instead of MCRs.

\begin{example}
\label{ex:MCR3}
Returning to Example~\ref{ex:MCR2}, $K_2 \myent_1 p$ and $K_2 \myent_1^B p$ hold because $\rho_{\thicksim_1^c}(K_2) \vdash x_1$ and $\rho_{\thicksim_2^c}(K_2) \vdash x_1$. However, $K_2\not\myent_2 p$ and $K\not\myent_2^B p$ hold since $\rho_{\thicksim_1^c}(K_2) \nvdash x_2$ and $\rho_{\thicksim_2^c}(K_2) \nvdash x_2$. Additionally, both $K_2 \myent_1^B q$ and $K_2 \myent_2^B q$ hold because $\rho_{\thicksim_1^c}(K_2) \vdash q$, and $\thicksim_1^c$ is the unique BMCR of $K_2$.
Using $\rho_{\thicksim_1^c}(K_2) \nvdash q$, we obtain $K_2 \not\myent_1 q$ and $K_2 \not\myent_2 q$.
\end{example}

Note that while our analysis focuses on a limited set of principles to define IRs, our approach can be adapted to establish numerous other IRs by employing principles similar to those used in the case of MCSes (e.g., see~\cite{PinkasL92}). For instance, one can derive alternative IRs by considering conclusions derived  from at least one MCR or from preferred MCRs, taking into account different preference criteria, such as prioritizing the largest MCRs.

Since the unique (B)MCR of a consistent PB $K$ is $\thicksim_c^K$, we deduce that our four IRs coincide with the classical  IR in the case of consistent PBs. 
\begin{proposition}
For every consistent PB $K$ and every formula $\phi$, the following properties  are equivalent:  $K \vdash \phi$, $K \myent_1 \phi$, $K \myent_2 \phi$, $K \myent_1^B \phi$, and $K \myent_2^B \phi$.
\end{proposition}

Next, we  examine the relationships between the introduced IRs.
We can clearly see the following: $\myent_2 \subseteq \myent_1$, $\myent_2^B \subseteq \myent_1^B$, $\myent_1 \subseteq \myent_1^B$, and $\myent_2 \subseteq \myent_2^B$. Further, using the properties related to the conclusion $p$ in Example~\ref{ex:MCR3}, we establish that $\myent_2 \subsetneq \myent_1$, $\myent_2^B \subsetneq \myent_1^B$, and $\myent_1 \not\subseteq \myent_2^B$. By considering the properties related to the conclusion $q$, we find that $\myent_1 \subsetneq \myent_1^B$, $\myent_2 \subsetneq \myent_2^B$, and $\myent_2^B \not\subseteq \myent_1$.
Consequently, $\myent_2$ is more cautious than the other three  IRs, whereas  $\myent_1^B$ is less cautious than the remaining  three IRs.

 Now, we exhibit some relationships between our IRs and Priest's Minimally Inconsistent Logic of Paradox ($LP_m$)~\cite{Priest91}.

 An $LP_m$ interpretation is a function $\lambda$ that assigns a value in $\{\{0\}, \{1\},\{0,1\}\}$ to each formula in  $\textsf{PF}$ and meets the following conditions:  
 $\lambda(\neg\phi)=\{1-v : v\in \lambda(\phi)\}$ and $\lambda(\phi\wedge\psi)=\{v\times v' : v\in \lambda(\phi), v'\in \lambda(\psi)\}$.
  We use $\lambda!$ to denote the set of variables $p$ such that $\lambda(p)=\{0,1\}$.
 
 We say that $\lambda$ is an $LP_m$ model of $\phi$, written $\lambda\models_{LP_m}\phi$,  if and only if  $1\in \lambda(\phi)$.

 An $LP_m$ model of a formula $\phi$ is said to be \emph{minimal} iff, for any  $LP_m$ model  $\lambda'$ of $\phi$, it does not hold that $\lambda'!\subsetneq \lambda!$.
 
 A PB $K$ entails  $\phi$ in $LP_m$, written $K\vdash_{LP_m}\phi$, iff 
 for every minimal $LP_m$ model $\lambda$ of $\bigwedge K$,  $\lambda\models_{LP_m}\phi$ holds. 

 First, we have   $\{p,\neg p\}\not\myent_{1}^B p\wedge\neg p$ and $\{p,\neg p\}\vdash_{LP_m} {p\wedge\neg p}$. Then, we have $\vdash_{LP_m}\not \subseteq \myent_1^B$, which implies  $\vdash_{LP_m}\not \subseteq \myent_1$, $\vdash_{LP_m}\not \subseteq \myent_2$, and $\vdash_{LP_m}\not \subseteq \myent_2^B$.

 Furthermore, both $ \myent_1^B\not\subseteq \vdash_{LP_m}$ and $ \myent_2^B\not\subseteq \vdash_{LP_m}$ hold. This is demonstrated by the fact that 
 $ \{p,\neg p, \neg p\vee q\}\myent_1^B q$ and $ \{p,\neg p, \neg p\vee q\}\myent_2^B q$, whereas $ \{p,\neg p, \neg p\vee q\}\nvdash_{LP_m} q$.
  
Given an $LP_m$ model $\lambda$ of $K$, $\sim_\lambda$ represents an  equivalence relation on  $Occ(K)$ defined by 
${Occ(K)/\sim_\lambda}= \{Occ(p,K) : |\lambda(p)|=1\}\cup \{PosOcc(p,K), $ $NegOcc(p,K) : \lambda(p)=\{0,1\}\}$.

 \begin{proposition}
\label{prop:lpm1}
For any formula $\phi$, for any $LP_m$ interpretation,  and for any truth value $v$,
if $v\in \lambda(\phi)$, then $v\in \lambda_{p\mapsto \{0,1\}}(\phi)$ for every $p\in\textsf{var}(\phi)$.
\end{proposition}
\begin{proof}
We proceed by structural induction on the formula $\phi$. 
If $\phi=p$, then  $\lambda_{p\mapsto \{0,1\}}(\phi)=\{0,1\}$ by definition; hence,  it holds that  $v\in \lambda_{p\mapsto \{0,1\}}(\phi)$.
If $\phi=q$ where $p\neq q$, we have $v\in \lambda_{p\mapsto \{0,1\}}(\phi)$ since $\lambda_{p\mapsto \{0,1\}}(q)=\lambda(q)$.
Consider now the inductive cases.
If $\phi=\psi\wedge\chi$ and $v=1$, then $v\in  \lambda(\psi)$ and $v\in  \lambda(\chi)$. By the induction hypothesis, 
 $v\in \lambda_{p\mapsto \{0,1\}}(\psi)$ and $v\in \lambda_{p\mapsto \{0,1\}}(\chi)$ hold.This implies that  $v\in \lambda_{p\mapsto \{0,1\}}(\phi)$. 
 Similarly, if $\phi=\psi\wedge\chi$ and $v=0$, then $v\in  \lambda(\psi)$ or $v\in  \lambda(\chi)$. Using the induction hypothesis, 
 $v\in \lambda_{p\mapsto \{0,1\}}(\psi)$ or $v\in \lambda_{p\mapsto \{0,1\}}(\chi)$ holds, leading to $v\in \lambda_{p\mapsto \{0,1\}}(\phi)$. 
 If $\phi=\neg\psi$, then $1-v\in \lambda(\psi)$. By the induction hypothesis, we obtain $1-v\in \lambda_{p\mapsto \{0,1\}}(\psi)$.
Therefore,  $v\in \lambda_{p\mapsto \{0,1\}}(\phi)$ holds. 
\end{proof}

Given two formulas $\phi$ and $\psi$,  and a variable occurrence $o$ in $\phi$, we use $\phi[o/\psi]$ to denote the formula obtained from $\phi$ by replacing $o$ with $\psi$.

 \begin{proposition}
\label{prop:lpm12}
Let $\phi$ be a propositional formula, $\lambda$ an $LP_m$ interpretation,  and $r$ a variable that does not occur in $\phi$. The following properties hold: 
\begin{enumerate}
\item If $o$ is a positive occurrence  in $\phi$ and $1\in \lambda(\phi)$, then $1\in \lambda_{|r\mapsto 1}(\phi[o/r])$. 
\item If $o$ is a negative  occurrence  in $\phi$  and $0\in \lambda(\phi)$, then $0\in \lambda_{|r\mapsto 1}(\phi[o/r])$.   
\item If $o$ is a negative occurrence  in $\phi$  and $1\in \lambda(\phi)$, then $1\in \lambda_{|r\mapsto 0}(\phi[o/r])$. 
\item If $o$ is a positive  occurrence  in $\phi$  and $0\in \lambda(\phi)$, then $0\in \lambda_{|r\mapsto 0}(\phi[o/r])$.   
\end{enumerate}
\end{proposition}
\begin{proof}
We focus only  on  the case of the first two properties, the  others being  similar. 
The proof is by mutual induction on the structure of $\phi$. The base case is straightforward: if $\phi$ corresponds to a variable $p$, 
then there are no negative occurrences of any variables; if $1\in \lambda(\phi)$,  then $1\in \lambda_{|r\mapsto 1}(r)$ by definition.
For the inductive  case $\phi=\psi\wedge \chi$, we only need the induction hypothesis. Indeed,  we have $1\in \lambda(\phi)$ iff $1\in \lambda(\psi)$ and $1\in \lambda(\chi)$, and 
$0\in \lambda(\phi)$ iff $0\in \lambda(\psi)$ or $0\in \lambda(\chi)$. Furthermore, $o$ is positive (resp. negative) in $\phi$ iff  $o$ is positive (resp. negative) in either 
$\psi$ or $\chi$. Consider now the inductive case $\phi=\neg\psi$. Let  us start with $o$ being  a  positive  occurrence in $\phi$ and $1\in \lambda(\phi)$. 
Thus, $0\in \lambda(\psi)$ holds. Using the fact that $o$ is negative in $\psi$ and applying the induction hypothesis, we obtain 
 $0\in \lambda_{|r\mapsto 1}(\psi[o/r])$, which implies  $1\in \lambda_{|r\mapsto 1}(\neg\psi[o/r])$. 
 Now, if  $o$ is  a  negative  occurrence in $\phi$ and $0\in \lambda(\phi)$, then  $1\in \lambda(\psi)$ holds.
 Using the fact that $o$ is positive in $\psi$ and applying the induction hypothesis, we deduce 
 $1\in \lambda_{|r\mapsto 1}(\psi[o/r])$, leading to  $0\in \lambda_{|r\mapsto 1}(\neg\psi[o/r])$. 
 \end{proof}


The next proposition highlights that an MCR can be derived from every  model  in $LP_m$.
  \begin{proposition}
  \label{prop:int1}
  If $\lambda$ is a minimal $LP_m$ model of $K$, then $\sim_\lambda$  is an MCR of $K$.
  \end{proposition}
    \begin{proof}
Assume that $\lambda$ is a minimal $LP_m$model of $K$. Clearly, $\thicksim_\lambda$ satisfies Compliance. Additionally, given that  $\lambda$ is an $LP_m$ model of $K$, it follows that $\thicksim_\lambda$ satisfies Consistency. Indeed, using  mainly Proposition~\ref{prop:lpm12},  a model $\omega_\lambda$ of 
$\bigwedge\textsf{R}(K)\wedge \bigwedge_{(o,o')\in \thicksim_\lambda} (\textsf{R}(o)\leftrightarrow \textsf{R}(o'))$ can be obtained by meeting the following conditions: 
\begin{itemize}
\item $\omega_\lambda(\textsf{R}(o)) = 0$ if $o$ is an occurrence of $p$ and either $\lambda(p) = \{0\}$ or $o$ is negative and $\lambda(p) = \{0,1\}$;
\item $\omega_\lambda(\textsf{R}(o)) = 1$ if $o$ is an occurrence of $p$ and either $\lambda(p) = \{1\}$ or $o$ is positive and $\lambda(p) = \{0,1\}$.
\end{itemize}

Suppose for contradiction that $\sim_\lambda$ does not satisfy Maximality. Then there exists an MCR $\sim'$ of $K$ such that $\sim_\lambda \subsetneq \sim'$. This implies the following:
\begin{enumerate}
\item $\{Occ(p, K) :  |\lambda(p)|=1\} \subseteq {Occ(K)/\sim'}$.
\item For every variable $p$ occurring in $\{PosOcc(p, K), NegOcc(p, K) \mid \lambda(p) = \{0,1\}\}$, either $Occ(p, K) \in {Occ(K)/\sim'}$, or $PosOcc(p, K) \in {Occ(K)/\sim'}$ and $NegOcc(p, K) \in {Occ(K)/\sim'}$.
\item There exists a variable $p$ such that $PosOcc(p, K) \in {Occ(K)/\sim}$, $NegOcc(p, K) \in {Occ(K)/\sim}$, and $Occ(p, K) \in {Occ(K)/\sim'}$.
\end{enumerate}
Let $\omega$ be a model of $\Psi = \textsf{R}(\phi) \wedge \bigwedge_{(o,o') \in \sim'} (\textsf{R}(o) \leftrightarrow \textsf{R}(o'))$. We define an $LP_m$ interpretation $\lambda'$ as follows: $\lambda'(p) = \{\omega(o)\}$ if $|EqC(p, \sim')| = 1$ and $o$ is an arbitrary occurrence of $p$; otherwise, $\lambda'(p) = \{0,1\}$.
Using Properties (1), (2), and (3), we obtain $\lambda'! \subsetneq \lambda!$. Therefore, $\lambda$ is not a minimal $LP_m$ model, leading to a contradiction.
\end{proof}

Considering  that  $\vdash_{LP_m}\not \subseteq \myent_1$, the following theorem shows that $\myent_1$ is more cautious than $\vdash_{LP_m}$.

  \begin{theorem}
  For every PB $K$ and every formula $\phi$, if ${K\myent_1\phi}$, then $K\vdash_{LP_m}\phi$.
  \end{theorem}
  \begin{proof}
Assume that $K \myent_1 \phi$. Then, for any MCR $\thicksim$, there exists a tuple $(q_1, \ldots, q_n) \in \textsf{P}(\rho_\thicksim, S)$ such that $\rho_\thicksim(K) \vdash \phi[p_1/q_1, \ldots, p_m/q_m]$, where $S = (p_1, \ldots, p_m)$ and $\textsf{var}(K) \cap \textsf{var}(\phi) = \{p_1, \ldots, p_m\}$. Let $\lambda$ be a minimal $LP_m$ model of $K$. By Proposition~\ref{prop:int1}, $\sim_\lambda$ is an MCR of $K$. Therefore, there exists a tuple $(q_1, \ldots, q_n) \in \textsf{P}(\rho_{\thicksim_\lambda}, S)$ such that $\rho_{\thicksim_\lambda}(K) \vdash \phi[p_1/q_1, \ldots, p_m/q_m]$.

Let $\omega$ be a Boolean interpretation such that:
\begin{itemize}
\item $\omega(\rho_{\sim_\lambda}(o)) = 0$ if $o$ is an occurrence of $p$ and either $\lambda(p) = \{0\}$ or $o$ is negative and $\lambda(p) = \{0,1\}$;
\item $\omega(\rho_{\sim_\lambda}(o)) = 1$ if $o$ is an occurrence of $p$ and either $\lambda(p) = \{1\}$ or $o$ is positive and $\lambda(p) = \{0,1\}$.
\end{itemize}
Using Proposition~\ref{prop:lpm12}, the interpretation $\omega$ is a model of $\rho_{\thicksim_\lambda}(K)$. 
Thus, $\omega \models \phi[p_1/q_1, \ldots, p_m/q_m]$ holds. Let $X= \{p_i\in \{p_1, \ldots, p_m\} : q_i\neq p_i\}$. Define an  $LP_m$ interpretation $\lambda'$ as follows: for every $p \notin X$, $\lambda'(p) = \{\omega(p)\}$, and for every $p \in X$, $\lambda'(p) = \{0,1\}$.
Using  Proposition~\ref{prop:lpm1} with $\omega \models \phi[p_1/q_1, \ldots, p_m/q_m]$,  we obtain that $\lambda'$ satisfies $\phi$. However, we know that for any variable $p_i$ occurring in $S$, we have $q_i \neq p_i$ if and only if $\lambda(p_i) = \{0,1\}$. Thus, for every variable $p$ in $\phi$, it holds that $\lambda'(p)=\lambda(p)$.
We conclude that $\lambda$ satisfies  $\phi$.
\end{proof}

 \section{Occurrence-based Semantics} 
 Building on the approach  of differentiating between occurrences  of the same variable to restore consistency,  
 we introduce an unusual semantics that assigns truth values to the  occurrences of variables rather than to the variables themselves. 
The entailment is established through Boolean interpretations that align with the occurrence-based models.

An \emph{occurrence-based interpretation} (o-interpretation for short) of a formula $\phi$ is a function $\mu$ mapping each occurrence in $Occ(\phi)$ to either $0$ or $1$. 
We say that $\mu$ is an \emph{o-model} of $\phi$ if and only if $\omega_\mu \models \textsf{R}(\phi)$, where $\omega_\mu$ is any  boolean interpretation such that, for each $o\in Occ(\phi)$, $\omega_\mu (\textsf{R}(o))= \mu(o)$. An o-model of a PB $K$ is an o-model of its corresponding formula $\bigwedge K$.

To define our IRs in the framework of  occurrence-based semantics, we examine two minimality properties in o-models: a-minimality and b-minimality.

We denote by  $\textsf{diff}_a(\mu)$  the set of ordered pairs $\{(o, o') \in Occ(\phi) \times Occ(\phi) : \textsf{var}(o) = \textsf{var}(o'), \mu(o) \neq \mu(o')\}$. 
We then define the preorder relation $\preceq_a$ on the o-models of $\phi$ such that $\mu \preceq_a \mu'$ if and only if $\textsf{diff}_a(\mu) \subseteq \textsf{diff}_a(\mu')$. The corresponding strict preorder is denoted by $\prec_a$.

An o-model $\mu$ of $\phi$ is considered \emph{a-minimal} if it is minimal with respect to $\preceq_a$, i.e., for any o-interpretation $\mu'$ of $\phi$ where $\mu' \prec_a \mu$, $\mu'$ is not an o-model of $\phi$. 

Observe that for every a-minimal o-model $\mu$ of a consistent formula, $\textsf{diff}_a(\mu) = \emptyset$. This indicates that every a-minimal o-model in such cases can be regarded as a Boolean interpretation: all occurrences of each variable have the same truth value.

A relationship between MCRs and a-minimal o-models is established in the following proposition.
\begin{proposition} 
\label{prop:right1}
For every formula $\phi$, 
if $\mu$ is an a-minimal o-model of $\phi$, then  $\thicksim_t^\phi\setminus \textsf{diff}_a(\mu)$ is an MCR of $\phi$.
\end{proposition}
\begin{proof}
First, we establish $(o,o')$ belongs to  $\thicksim=\thicksim_t^\phi\setminus \textsf{diff}_a(\mu)$ iff $\mu(o)=\mu(o')$. This relation is an equivalence relation as the equality operation inherently satisfies reflexivity, symmetry, and transitivity.
Next, given that $\mu$ satisfies $\phi$, it follows that $R(\phi)\wedge \bigwedge_{(o,o')\in \thicksim} 
(\textsf{R}(o)\leftrightarrow \textsf{R}(o'))$ is consistent.  Thus, $\thicksim$  upholds the property of  Consistency. Finally, 
the minimality of $\mu$ w.r.t. $\preceq_a$ implies the maximality of $\thicksim$ w.r.t. set inclusion.
\end{proof}

Given an MCR $\thicksim$ of $\phi$ and a model $\omega$ of $\Psi=\textsf{R}(\phi)\wedge \bigwedge_{(o,o')\in \thicksim}  (\textsf{R}(o)\leftrightarrow \textsf{R}(o'))$, 
we define $\mu_\omega$ as an  o-interpretation such that  $\mu(o)=1$ if and only if $\omega(\textsf{R}(o))=1$. Then, 
the set  $\textsf{OM}(\thicksim)$ consists of the o-interpretations $\{\mu_\omega : \omega\in\textsf{mod}(\Psi)\}$.

\begin{proposition} 
\label{prop:left1}
For every formula $\phi$, 
if $\thicksim$ is an MCR of $\phi$, then $\textsf{OM}(\thicksim)$ is a set of a-minimal o-models of $\phi$. 
\end{proposition}
\begin{proof}
Let $\mu_\omega$ in $\textsf{OM}(\thicksim)$. Owing to the fact that $\omega$ is a model of $\textsf{R}(\phi)$, $\mu$ is an o-model of $\phi$.
Moreover,  we have  $\omega(\textsf{R}(o))\neq \omega(\textsf{R}(o'))$ for any $(o,o')\in {\thicksim_t^\phi \setminus \thicksim}$. Hence ${\thicksim_t^\phi\setminus \thicksim=\textsf{diff}_a(\mu_\omega)}$ holds. 
Using  the maximality of $\thicksim$, we obtain the minimality of $\textsf{diff}_a(\mu_\omega)$.
\end{proof}

We say that a Boolean interpretation $\omega$ is \emph{compatible with an o-interpretation} $\mu$ of $\phi$ if and only if 
for every propositional variable $p$ occurring in $\phi$, there exists an occurrence $o$ of this variable  such that $\omega(p)=\mu(o)$. 

We define the IRs $\myent_{a1}$ and $\myent_{a2}$ as follows:
\begin{itemize}
\item  $K\myent_{a1} \phi$ iff for each a-minimal o-model $\mu$ of $K$, there exists a Boolean interpretation  $\omega$ that  is compatible with $\mu$  
such that  $\omega\models\phi$.
\item  $K\myent_{a2} \phi$ iff for each a-minimal o-model $\mu$ of $\phi$, and for each Boolean interpretation  $\omega$  which  is compatible with $\mu$,   
 it follows that $\omega\models\phi$. 
\end{itemize}

Using mainly Proposition~\ref{prop:right1}, we derive the following property.

\begin{proposition}
  For every PB $K$ and every formula $\phi$, $if K\myent_{1}\phi$, then $K\vdash_{a1}\phi$.
\end{proposition}

In fact,  $\myent_{1}$ is more cautious than  $\myent_{a1}$. We can illustrate this using  $K=\{p, \neg p, q\vee r\}$ and $\phi=(\neg p\wedge (\neg q\vee \neg r))\vee (p\wedge q\wedge r)$. Indeed, we have $K\myent_{a1} \phi$ but $K\not\myent_{1} \phi$. 

The distinction between   $\myent_{a1}$ and  $\myent_{1}$ mainly arises from the fact that  each MCR $\myent_{1}$ necessitates finding an adaptation of the conclusion that can be entailed, while $\myent_{a1}$ requires an adaptation for each o-interpretation corresponding to an MCR. For instance, the previous PB $K$ admits a single MCR 
$\sim$, defined by ${Occ(K)/\sim} =\{\{p_1^+\}, \{p_2^-\}, \{q_1^+\}, \{r_1^-\}\}$.  It does not hold that $K\myent_{1} \phi$ because 
$\{x_1, \neg x_2, q\vee r\}$ does not entail any of the possible adaptations  $(\neg x_1\wedge (\neg q\vee \neg r))\vee (x_1\wedge q\wedge r)$ and 
$(\neg x_2\wedge (\neg q\vee \neg r))\vee (x_2\wedge q\wedge r)$.  We have three  a-minimal o-models 
$\mu_1=\{p_1^+\mapsto 1, p_2^-\mapsto 0, q_1^+\mapsto 1, r_1^-\mapsto 0\}$, $\mu_2=\{p_1^+\mapsto 1, p_2^-\mapsto 0, q_1^+\mapsto 0, r_1^-\mapsto 1\}$ and 
$\mu_3=\{p_1^+\mapsto 1, p_2^-\mapsto 0, q_1^+\mapsto 1, r_1^-\mapsto 1\}$.  The conclusion $\phi$ can be derived  using three respective interpretations $\omega_1$, 
 $\omega_2$, and  $\omega_3$, which satisfy  the following conditions:  
$\{ p\mapsto 0, q\mapsto 1, r\mapsto 0\}\subseteq \omega_1$, $\{ p\mapsto 0, q\mapsto 0, r\mapsto 1\}\subseteq \omega_2$,  and 
$\{p\mapsto 1, q\mapsto 1, r\mapsto 1\}\subseteq \omega_3$.

Interestingly,  the IR $\myent_{a2}$ coincides with $\myent_{2}$. 

\begin{proposition}
  For every PB $K$ and any formula $\phi$, $K\myent_{a2}\phi$ iff $K\vdash_{2}\phi$.
\end{proposition}
\begin{proof}
Let $S=(p_1,\ldots{}, p_m)$ s.t. 
$ \textsf{var}(\phi)\cap \textsf{var}(K)=\{p_1,\ldots{}, p_n\}$.

Let us first consider the \emph{if} part. Assume that $K\vdash_{2}\phi$. Let $\mu$ be an a-minimal o-model of $K$. Then, using Proposition~\ref{prop:right1},  we know that 
${\thicksim= {\thicksim_t^\phi\setminus \textsf{diff}(\mu)}}$ is an MCR.  
Let $\omega$ be a Boolean interpretation s.t. $\omega(\rho_\thicksim (o))=\mu(o)$ for every occurrence $o$. 
Clearly, the fact that $\mu$ is an o-model of $K$ implies that $\omega$ is also a model of $\rho_\thicksim (K)$.  Let $\omega'$ be an interpretation compatible with $\mu$.  Clearly, 
for every $p_i\in S$, there exists $q_i\in\lceil\rho_\thicksim \rceil(p_i)$ s.t. $\omega'(p_i)=\omega(q_i)$. 
Thus, using $\omega\models  \phi[p_1/q_1,\ldots{}, p_m/q_m] $, we deduce, we deduce $\omega'\models \phi$. 
 
 Consider now the \emph{only if} part. Suppose that $K\myent_{2a}\phi$. Let $\sim$ be an MCR of $K$,  $\omega$ a model of $\rho_\thicksim(K)$,  and 
$(q_1,\ldots{}, q_n)\in \textsf{P}(\rho_\thicksim, S)$.  Let $\mu$ be the o-interpretation defined as follows:  $\mu(o)=1$ if and only if $\omega(\rho_\thicksim(o))=1$. 
One can se that  $\mu$ belongs to $\textsf{OM}(\thicksim)$. This leads  to $\mu$ is an a-minimal o-model of $K$ (see Proposition~\ref{prop:left1}).
Thus,  any Boolean interpretation compatible  with $\mu$ is a model of $\phi$. 
Let $\omega'$ a Boolean interpretation s.t. $\omega'(p_i)=\omega(q_i)$. Clearly $\omega'$ is compatible with $\mu$, leading to $\omega'\models \phi$. 
Consequently, it holds that $\omega\models  \phi[p_1/q_1,\ldots{}, p_m/q_m] $.
\end{proof}

Given an o-interpretation  $\mu$ of $K$, we use $\lambda_\mu$ to denote an $LP_m$ interpretation  defined as follows: (1)~for every variable $p$ occurring in $K$, 
 $1\in \lambda_\mu(p)$ iff there exists $o\in Occ(p, K)$ s.t. $\mu(o)=1$, and $0\in \lambda_\mu(p)$ iff there exists $o\in Occ(p, K)$ s.t. $\mu(o)=0$; 
 (2)~for any variable $p$ not in $\textsf{var}(K)$, $\lambda(p)=\{0\}$ (however, this choice of truth value is arbitrary and could be set to any truth value).

\begin{proposition}
\label{prop:lpm2}
For any formula $\phi$, if $\mu$ is an o-model of $\phi$, then  $\lambda_\mu$ is an $LP_m$ model of $\phi$.
\end{proposition}
\begin{proof}
Suppose that $\mu$ is an o-model of $\phi$. Then $\phi$ admits a model $\omega$  compatible with  $\mu$. 
Clearly, there exist variables $p_1,\ldots{}, p_k$ such that $\lambda_\mu=\omega_{|p_1\mapsto \{0,1\},\ldots{},p_k\mapsto \{0,1\}}$.
Thus, using Proposition~\ref{prop:lpm1}, we obtain that $1\in \lambda_\mu(\phi)$ since $1\in \omega(\phi)$.
\end{proof}

Let us observe that $\mu$ can be an a-minimal o-model while $\lambda_\mu$ is not a minimal $LP_m$ model. 
Consider, for instance, the KB $K=\{p , p\rightarrow \neg p\wedge q, p\rightarrow \neg q\}$. 
One of its MCR is $\thicksim$ defined by ${Occ(K)/\thicksim} = \{\{p_1^+, p_2^-, p_4^-\},\{p_3^-\}, \{q_1^+\},  \{q_2^-\}\}$. 
Thus, any $\mu$ in $\textsf{OM}(\thicksim)$ is an a-minimal o-model. 
However, $\{p\mapsto \{0,1\}, q\mapsto \{0,1\}\}$ is not a minimal $LP_m$ model because $\{p\mapsto \{0,1\}, q\mapsto 1\}$  is an $LP_m$ model.

Given an $LP_m$ interpretation $\lambda$ and a PB $K$, we define $\mu_\lambda^K$  as  the o-interpretation over $Occ(K)$ using the following criteria: for each occurrence $o=\langle p, \phi, i\rangle$, if $\lambda(p)=\{0\}$, then $\mu_\lambda^K(o)= 0$; if $\lambda(p)=\{1\}$, then $\mu_\lambda^K(o)= 1$; Otherwise, if $\lambda(p)=\{0,1\}$, then $\mu_\lambda^K(o)=0$ if $o$ is a negative occurrence,   and $\mu_\lambda^K(o)=1$ if $o$ is positive.

\begin{proposition}
\label{prop:lambda1}
Let $K$ be a PB and $\lambda$ an  $LP_m$ interpretation s.t. $|\lambda(p)|=1$ for every variable $p\notin\textsf{var}(K)$. Then, 
$\lambda$ is a minimal $LP_m$ model of $K$ iff  $\mu_\lambda^K$ is an a-minimal o-model of $K$. 
\end{proposition}
\begin{proof}
First, note that, using mainly Proposition~\ref{prop:lpm12}, we obtain  that if $\lambda$ is an $LP_m$ model of $\phi$, then $\mu_\lambda^K$ is an o-model of $\phi$.

Consider the \emph{if} part. Assume that $\mu_\lambda^K$ is a minimal o-model of $K$. 
By applying Proposition~\ref{prop:lpm2}, we  conclude  that $\lambda$  is an $LP_m$ model of $K$. 
Suppose for contradiction that $\lambda$ is not a minimal $LP_m$ model. 
Then there exists another $LP_m$ model $\lambda'$ of $\phi$ s.t. $\lambda'!\subsetneq \lambda!$.  This implies that $\mu_{\lambda'}^K$ is an o-model of $K$ and   $\mu_{\lambda'}^K\prec_a \mu_{\lambda}^K$, leading to a contradiction. Consequently, $\lambda$ is a minimal $LP_m$ model of $K$.

Let us consider the \emph{only if} part. 
Suppose that $\lambda$ is a minimal $LP_m$ model of $K$. Suppose for contradiction that $\mu_\lambda^K$ is not an a-minimal o-model of $K$.
Then there exists an o-model $\mu$ of $K$ s.t. $\mu\prec_a\mu_\lambda^K$. Then, (1)~for each variable $p$, if $\lambda(p)$ consists of single truth value, 
then $\mu(o)$  assigns a unique truth value to all occurrence $o$ of $p$; (2)~there exists a variable $p$ and a truth value $v$ s.t. $\lambda(p)=\{0,1\}$ and $\mu(o)=v$ for all occurrence $o$ of $p$. Using Proposition~\ref{prop:lpm2}, we obtain that $\lambda_\mu$  is a $LP_m$ model of $K$. Moreover, 
 using Properties (1) and (2), we obtain that $\lambda_\mu!\subsetneq \lambda!$. This results in  a contradiction. 
 Therefore, $\mu_\lambda^K$ is  an a-minimal o-model of $K$.
\end{proof}

The forgoing proposition leads to the following theorem. 

\begin{theorem}
For every PB $K$ and any formula $\phi$, if $K\myent_{a1}\phi$ then $K\vdash_{LP_m}\phi$.
\end{theorem}

Given the previous theorem and the fact that $\{p,\neg p\}\vdash_{LP_m} p\wedge\neg p$ and $\{p,\neg p\}\not\myent_{a1} p\wedge\neg p$, we conclude that 
$\myent_{a1}$ is more cautious than $\vdash_{LP_m}$.

IRs similar to $\myent_1^B$ and $\myent_2^B$ can also be defined in our  framework of  occurrence-based semantics. We use $\textsf{diff}_b(\mu)$ to denote the set of ordered pairs $\{(o, o') \in PosOcc(\phi) \times NegOcc(\phi) : \textsf{var}(o) = \textsf{var}(o'), \mu(o) \neq \mu(o')\}$. We  define a preorder relation $\preceq_b$ on the o-models of $\phi$ as follows: $\mu \preceq_b \mu'$ if and only if $\textsf{diff}_b(\mu) \subseteq \textsf{diff}_b(\mu')$. Its associated  strict preorder is denoted by $\prec_b$.

An o-model $\mu$ of $\phi$ is said to be  \emph{b-minimal} if and only if it is a-minimal and, for any o-model $\mu'$ of $\phi$,  $\mu'\not\prec_b\mu$ holds.

The IRs $\myent_{b1}$ and $\myent_{b2}$ are defined in the same way as  $\myent_{a1}$ and $\myent_{a2}$, respectively, but using b-minimal o-model instead 
a-minimal o-models. 

The relation $\myent_{b2}$ also coincides with $\myent_2^B$. 
Moreover, 
since every  b-minimal o-model is also a-minimal, it follows that  $\myent_{a1}\subseteq\myent_{b1}$. 
Using $\{p,\neg p,\neg p\vee q\}\myent_{b1} q$, $\{p,\neg p,\neg p\vee q\}\not\myent_{a1} q$ and   $\{p,\neg p,\neg p\vee q\}\not\vdash_{LP_m} q$,   
we deduce $\myent_{a1}\subsetneq\myent_{b1}$ and  $\myent_{b1}\not\subseteq \vdash_{LP_m}$.

Let us summarize the cautiousness relationships: 
\begin{itemize}
\item  $\myent_2 \subsetneq \myent_1$, $\myent_2 \subsetneq \myent_2^B$, $\myent_1 \not\subseteq \myent_2^B$, $\myent_2^B \not\subseteq \myent_1$, $\myent_1 \subsetneq \myent_1^B$, $\myent_2^B \subsetneq \myent_1^B$,  $\myent_1 \subsetneq \vdash_{LP_m}$,    \\ $\myent_2^B \not\subseteq \vdash_{LP_m}$,  
 and  $ \vdash_{LP_m} \not\subseteq \myent_1^B$.
\item  $\myent_1 \subsetneq \myent_{a1}$, $\myent_{a2} = \myent_2$, $\myent_{1}^B \subsetneq \myent_{b1}$, $\myent_{b2} = \myent_2^B$,    $\myent_{a1} \subsetneq \myent_{b1}$,  
$ \myent_{a1} \not\subseteq \myent_{1}^B$,  $ \myent_{2}^B\not\subseteq \myent_{a1}$, 
$\myent_{a1} \subsetneq \vdash_{LP_m}$,
  and  $ \vdash_{LP_m} \not\subseteq \myent_{b1}$.
\end{itemize}

Since all our IRs are grounded in the classic inference relation and Boolean interpretations, they exhibit the following properties:
(1)~all consequences derived from our IRs are consistent; 
(2)~any formula equivalent to a derived consequence is itself also a consequence.

These two properties highlight the key distinctions between our approach and that of $LP_m$.
In fact, apart from $\vdash_{LP_m}$ leading to contradictions, a result motivated by dialetheism, $\vdash_{LP_m}$ can also yield a consistent formula without ensuring that its equivalent formulas are consequences as well. For instance, we have $\{p,\neg p, \neg q\}\vdash_{LP_m} (p\vee q)\wedge \neg p$ but $\{p,\neg p, \neg q\}\nvdash_{LP_m}  q\wedge \neg p$. 

\section{Conclusion and Perspectives}

We presented a variable occurrence-based framework for resolving propositional inconsistencies. We introduced two complimentary  concepts: Minimal Inconsistency Relations (MIRs) and Maximal Consistency Relations (MCRs). Using MCRs, we defined several non-explosive inference relations. 
We proposed further non-explosive inference relations using an occurrence-based semantics.

For future work, we aim to investigate  computational  problems associated with  MIRs,  MCRs, and  our inference relations.
Moreover, we intend  to develop additional inference relations by focusing on specific MCRs. 
Our plans also include broadening our approach to include some non-classical logics.

\bibliographystyle{splncs04}

\bibliography{biblio}

\begin{thebibliography}{10}
\providecommand{\url}[1]{\texttt{#1}}
\providecommand{\urlprefix}{URL }
\providecommand{\doi}[1]{https://doi.org/#1}

\bibitem{BaileyS05}
Bailey, J., Stuckey, P.J.: {Discovery of Minimal Unsatisfiable Subsets of
  Constraints Using Hitting Set Dualization}. In: Hermenegildo, M.V., Cabeza,
  D. (eds.) Practical Aspects of Declarative Languages, 7th International
  Symposium, {PADL} 2005, Long Beach, CA, USA, Proceedings. Lecture Notes in
  Computer Science, vol.~3350, pp. 174--186. Springer (2005)

\bibitem{BenferhatDP97}
Benferhat, S., Dubois, D., Prade, H.: {Some Syntactic Approaches to the
  Handling of Inconsistent Knowledge Bases: {A} Comparative Study Part 1: The
  Flat Case}. Studia Logica  \textbf{58}(1),  17--45 (1997)

\bibitem{benferhat:hal-03300219}
Benferhat, S., Dubois, D., Prade, H.: {Some syntactic approaches to the
  handling of inconsistent knowledge bases : A comparative study. Part 2 : the
  prioritized case}. In: {Logic at work: Essays Dedicated to the Memory of
  Helena Rasiowa}, Studies in Fuzziness and Soft Computing, vol.~24, pp.
  473--511. {Physica-Verlag, Heidelberg} (1999)

\bibitem{Besnard14}
Besnard, P.: {Revisiting Postulates for Inconsistency Measures}. In:
  Ferm{\'{e}}, E., Leite, J. (eds.) Logics in Artificial Intelligence - 14th
  European Conference, {JELIA} 2014, Funchal, Madeira, Portugal. Proceedings.
  Lecture Notes in Computer Science, vol.~8761, pp. 383--396. Springer (2014)

\bibitem{Besnard16}
Besnard, P.: {Forgetting-Based Inconsistency Measure}. In: Schockaert, S.,
  Senellart, P. (eds.) Scalable Uncertainty Management - 10th International
  Conference, {SUM} 2016, Nice, France. Proceedings. Lecture Notes in Computer
  Science, vol.~9858, pp. 331--337. Springer (2016)

\bibitem{BesnardH08}
Besnard, P., Hunter, A.: {Elements of Argumentation}. {MIT} Press (2008)

\bibitem{Brewka89}
Brewka, G.: {Preferred Subtheories: An Extended Logical Framework for Default
  Reasoning}. In: Sridharan, N.S. (ed.) Proceedings of the 11th International
  Joint Conference on Artificial Intelligence. Detroit, MI, USA. pp.
  1043--1048. Morgan Kaufmann (1989)

\bibitem{Gardenfors1992}
G{\"{a}}rdenfors, P. (ed.): {Belief Revision}. Cambridge Tracts in Theoretical
  Computer Science, Cambridge University Press (1992)

\bibitem{HunterK10}
Hunter, A., Konieczny, S.: On the measure of conflicts: {S}hapley
  {I}nconsistency {V}alues. Artificial Intelligence  \textbf{174},  1007--1026
  (2010)

\bibitem{KoniecznyMV19}
Konieczny, S., Marquis, P., Vesic, S.: {Rational Inference Relations from
  Maximal Consistent Subsets Selection}. In: Kraus, S. (ed.) Proceedings of the
  Twenty-Eighth International Joint Conference on Artificial Intelligence,
  {IJCAI} 2019. Macao, China. pp. 1749--1755. ijcai.org (2019)

\bibitem{LangM02}
Lang, J., Marquis, P.: {Resolving Inconsistencies by Variable Forgetting}. In:
  Fensel, D., Giunchiglia, F., McGuinness, D.L., Williams, M. (eds.)
  Proceedings of the Eights International Conference on Principles and
  Knowledge Representation and Reasoning, KR'02. Toulouse, France. pp.
  239--250. Morgan Kaufmann (2002)

\bibitem{LiffitonMPS05}
Liffiton, M.H., Moffitt, M.D., Pollack, M.E., Sakallah, K.A.: {Identifying
  Conflicts in Overconstrained Temporal Problems}. In: Kaelbling, L.P.,
  Saffiotti, A. (eds.) Proceedings of the Nineteenth International Joint
  Conference on Artificial Intelligence, IJCAI 2005. Edinburgh, Scotland, UK.
  pp. 205--211. Professional Book Center (2005)

\bibitem{Lin1994ForgetI}
Lin, F., Reiter, R.: {Forget It!} In: Proceedings of the AAAI Fall Symposium on
  Relevance. pp. 154--159 (1994)

\bibitem{PinkasL92}
Pinkas, G., Loui, R.P.: {Reasoning from Inconsistency: {A} Taxonomy of
  Principles for Resolving Conflict}. In: Nebel, B., Rich, C., Swartout, W.R.
  (eds.) Proceedings of the 3rd International Conference on Principles of
  Knowledge Representation and Reasoning, KR'92. Cambridge, MA, USA. pp.
  709--719. Morgan Kaufmann (1992)

\bibitem{Priest91}
Priest, G.: Minimally inconsistent {LP}. Studia Logica  \textbf{50},  321--331
  (1991)

\bibitem{sep-logic-paraconsistent}
Priest, G., Tanaka, K., Weber, Z.: {Paraconsistent Logic}. In: Zalta, E.N.
  (ed.) The {Stanford} Encyclopedia of Philosophy. Metaphysics Research Lab,
  Stanford University, {S}pring 2022 edn. (2022)

\bibitem{Reiter87}
Reiter, R.: {A Theory of Diagnosis from First Principles}. Artificial
  Intelligence  \textbf{32}(1),  57--95 (1987)

\bibitem{Rescher1970}
Rescher, N., Manor, R.: On inference from inconsistent premisses. Theory and
  Decision  \textbf{1},  179--217 (1970)

\bibitem{Tanaka13}
Tanaka, K., Berto, F., Mares, E.D., Paoli, F. (eds.): Paraconsistency: {L}ogic
  and {A}pplications, Logic, Epistemology, and the Unity of Science, vol.~26.
  Springer (2013)

\bibitem{Thimm:2018}
Thimm, M.: {On the Evaluation of Inconsistency Measures}. In: Grant, J.,
  Martinez, M.V. (eds.) {Measuring Inconsistency in Information}, Studies in
  Logic, vol.~73. College Publications (February 2018)

\end{thebibliography}

\end{document}